%% file: main.tex
\documentclass[twoside]{article}

\usepackage[table]{xcolor} 

\definecolor{colorABCOptimal}{rgb}{0.9, 1.0, 0.9} 
\definecolor{colorABOptimal}{rgb}{0.9, 0.9, 1.0} 
\definecolor{colorACOptimal}{rgb}{1.0, 0.9, 0.9} 

\usepackage{hyperref}       
\usepackage{url}            
\usepackage{booktabs}       
\usepackage{amsfonts}       
\usepackage{nicefrac}       
\usepackage{microtype}      
\usepackage{xcolor}         
\usepackage{amsmath}
\usepackage{amssymb}
\usepackage{mathtools}
\usepackage{amsthm}
\usepackage{algorithm}
\usepackage{algpseudocode}
\usepackage{wrapfig}
\usepackage{tikz}
\usetikzlibrary{arrows.meta, positioning, shapes, calc}
\usepackage{booktabs, tabularx}
\usepackage{multirow}
\usepackage{wrapfig}
\usepackage{booktabs}
\usepackage{array}
\usepackage[toc, page, header]{appendix}
\usepackage{hyperref}
\usepackage{siunitx}
\usepackage{makecell}
\DeclarePairedDelimiter{\abs}{\lvert}{\rvert}

\DeclarePairedDelimiter{\paren}{(}{)}

\DeclareMathOperator*{\E}{\mathbb{E}}

\let\Pr\relax 
\DeclareMathOperator*{\Pr}{\mathbb{P}}

\providecommand{\abs}[1]{\lvert#1\rvert}

\newcommand{\cD}{\mathcal{D}}

\newcommand{\cM}{\mathcal{M}}

\newcommand{\mat}[1]{{\mathbf #1}}

\def\Nset{\mathbb{N}}
\def\Rset{\mathbb{R}}

\renewcommand{\k}{\mat{k}}

\hypersetup{
  breaklinks   = true, 
  colorlinks   = true, 
  urlcolor     = blue, 
  linkcolor    = blue, 
  citecolor   = blue 
}

\theoremstyle{plain}
\newtheorem{theorem}{Theorem}[section]

\newtheorem{lemma}[theorem]{Lemma}
\newtheorem{corollary}[theorem]{Corollary}
\theoremstyle{definition}
\newtheorem{definition}[theorem]{Definition}

\theoremstyle{remark}

\usepackage[preprint]{aistats2026}

\usepackage[round]{natbib}

\begin{document}

%

%

\twocolumn[

\aistatstitle{Fast Inference via Hierarchical Speculative Decoding}

\aistatsauthor{ 
    Clara Mohri\textsuperscript{1,2} \And
    Haim Kaplan\textsuperscript{2,3} \And
    Tal Schuster\textsuperscript{4} \And
    Yishay Mansour\textsuperscript{2,3} \And
    Amir Globerson\textsuperscript{2,3}
}

\aistatsaddress{
    \textsuperscript{1}Harvard University \And
    \textsuperscript{2}Google Research \And
    \textsuperscript{3}Tel Aviv University \And
    \textsuperscript{4}Google DeepMind
} ]

\begin{abstract}
  Transformer language models generate text autoregressively, making inference latency proportional to the number of tokens generated. Speculative decoding reduces this latency without sacrificing output quality, by leveraging a small draft model to propose tokens that the larger target model verifies in parallel. In practice, however, there may exist a set of potential draft models—ranging from faster but less inaccurate, to slower yet more reliable. We introduce Hierarchical Speculative Decoding (HSD), an algorithm that stacks these draft models into a hierarchy, where each model proposes tokens, and the next larger model verifies them in a single forward pass, until finally the target model  verifies tokens. We derive an expression for the expected latency of any such hierarchy and show that selecting the latency-optimal hierarchy can be done in polynomial time. Empirically, HSD gives up to 1.2× speed-up over the best single-draft baseline, demonstrating the practicality of our algorithm in reducing generation latency beyond previous techniques.
\end{abstract}

\section*{Introduction}
\input{intro}

\input{relatedworks}

\input{background}

\input{hsd}
\input{hsd_optimization}

\input{experiments}

\input{conclusion}

\newpage

\bibliography{bibliography}
\bibliographystyle{abbrvnat}

\clearpage
\appendix
\thispagestyle{empty}
\onecolumn
\input{app_a}
\input{app_b}
\input{app_c}

\input{app_d}

\end{document}

%% file: intro.tex
Most language models are trained with teacher-forcing to predict the next token in an autoregressive fashion. While this allows for a highly parallelizable training process, inference remains a sequential process: a model must finish a full forward pass and predict a token before it can start processing the new context to predict the following token. This sequential process typically does not fully utilize the compute capabilities of modern accelerators, making text generation slow and costly.

Speculative decoding \citep{leviathan2023fast,chen2023accelerating} addresses this limitation by leveraging a smaller drafter model that autoregressively generates multiple tokens ahead. Then, these tokens are verified, and possibly discarded, by the larger target model in parallel with a single forward pass. By following the speculative sampling rejection rule, the output distribution of verified tokens is identical to that of the large model. Every round of drafting and verification yields at least one verified token in the worst case, and one more token than the number of drafted tokens in the best case.

Notably, there is a natural tradeoff in selecting the drafter for speculative decoding---a larger drafter will improve token acceptance rate but increase drafting latency. Many recent studies have investigated approaches for pushing the Pareto frontier of drafters~\citep{liu2023online, xiao2024parallelspec, zhang2024draftverifylossless, miao2024specinfer, hooper2023speed, zhou2023distillspec}. However, ultimately the practitioner may select the single drafter that provides the best accuracy-cost ratio. For example, when early exits from the target model are considered as drafters~\citep{elhoushi2024layer,NEURIPS2023_7b97adea,zhang2024draftverifylossless}, the layer that gives the best accuracy vs.\ depth tradeoff will be used.



In this paper, we question the paradigm of using only a single drafter. We study the prospect of leveraging multiple drafters in a cost-effective way. To this end, we introduce the Hierarchical Speculative Decoding (HSD) algorithm. In HSD, each drafter validates sequences generated by lower drafters in the hierarchy, and only the smallest drafter (i.e., lowest in hierarchy) generates autoregressively. We prove that using multiple drafters can further reduce latency while preserving the quality of the output.



Next, we turn to the question of finding the hierarchy which results in the optimal latency. A key challenge is that the number of possible hierarchies grows exponentially with the number of drafters available, and therefore naive enumeration would be costly. Furthermore, our algorithm has tunable parameters which should also be optimized. To address this, we derive an expression for the expected latency incurred by a given hierarchy and its parameters, and show that this expression can in fact be optimized in polynomial time. This is done via a reduction to the Generalized Shortest Path problem \citep{oldham2001combinatorial}, which admits a polynomial-time solution.



We validate the effectiveness of HSD empirically by implementing it on top of public open-source Large Language Models. Compared to both vanilla autoregressive decoding and to a single-drafter speculative decoding baseline, our method shows significant speedup gains. Hence, our main contributions are as follows:
\begin{enumerate}
    \item \textbf{Theoretical}: We introduce the Hierarchical Speculative Decoding algorithm for accelerating inference in LLMs and analyze its latency in Section~\ref{sec:hsd}. In Section~\ref{sec:optimization}, we formulate its corresponding optimization problem, and provide an efficient solution for optimal hierarchy construction.
    \item \textbf{Empirical}: In Section~\ref{sec:experiments}, we evaluate our method on open-source language models, and demonstrate speed up over speculative decoding with a single draft model.
\end{enumerate}


%% file: relatedworks.tex
\section*{Related Work}
\paragraph{Speculative decoding.} 
We build on the speculative decoding method~\citep{chen2023accelerating,leviathan2023fast} for accelerating transformers. In this framework, an efficient draft model generates tokens autoregressively, which are then verified in parallel by a target model using a sampling method that guarantees an identical output distribution to the target model. While some follow up work suggests new verification algorithms~\citep{liu2024parallel,narasimhan2025faster,sun2025block}, the vast majority of studies focus on improving the performance of drafters via techniques such as distillation~\citep{zhou2023distillspec}, enhanced attention to past verifier predictions~\citep{tandem}, multi-token prediction heads~\citep{cai2024medusa,mtp,eagle}, and other self-speculation solutions~\citep{zhang2024draftverifylossless} that further leverage signals from the target model. \citet{elhoushi2024layer} train a target model with an auxilliary early-exit loss \citep{Elbayad2020Depth,schuster2022confident} in order to obtain draft tokens from a post-hoc selected earlier layer in the target model.

Our work is complementary to most previous advancements, and presents a departure from the single drafter paradigm by replacing it with a hierarchy of drafters with increasing cost and accuracy.  
Perhaps most relevant is the work of \citet{sun2024triforce} that proposes a tailored two-stage hierarchy drafting method that leverages memory bottlenecks in certain deployment setups. In contrast, we introduce a general hierarchy framework with any set of appropriate drafter candidates, and develop an optimization solution for constructing the optimal hierarchy.

\paragraph{Hierarchical models.} Other uses of model hierarchies, ordered from weakest and cheapest to most capable and expensive, have demonstrated promising potential. One related domain is cascade models~\citep{deng2020cascaded, dohan2022language, gupta2024language, narasimhan2024faster} where typically the decision whether to use a larger model is based on a confidence measure over the prediction of the smaller model. Early exits in language models~\citep{bae2024relaxed,Elbayad2020Depth,schuster2022confident} can be viewed as a form of cascades that are nested within a single model. \citep{narasimhan2024faster} use a speculative decoding technique to perform deferral in a two-model cascade.
While these methods can provide promising speedups with quality guarantees in expectation, in contrast to speculative decoding, they do not give a per-example guarantee on the output distribution.

%% file: background.tex
\section{Background}
\label{sec:background}
\begin{figure*}[t]
\label{fig:alg}
\centering\includegraphics[width=350pt]{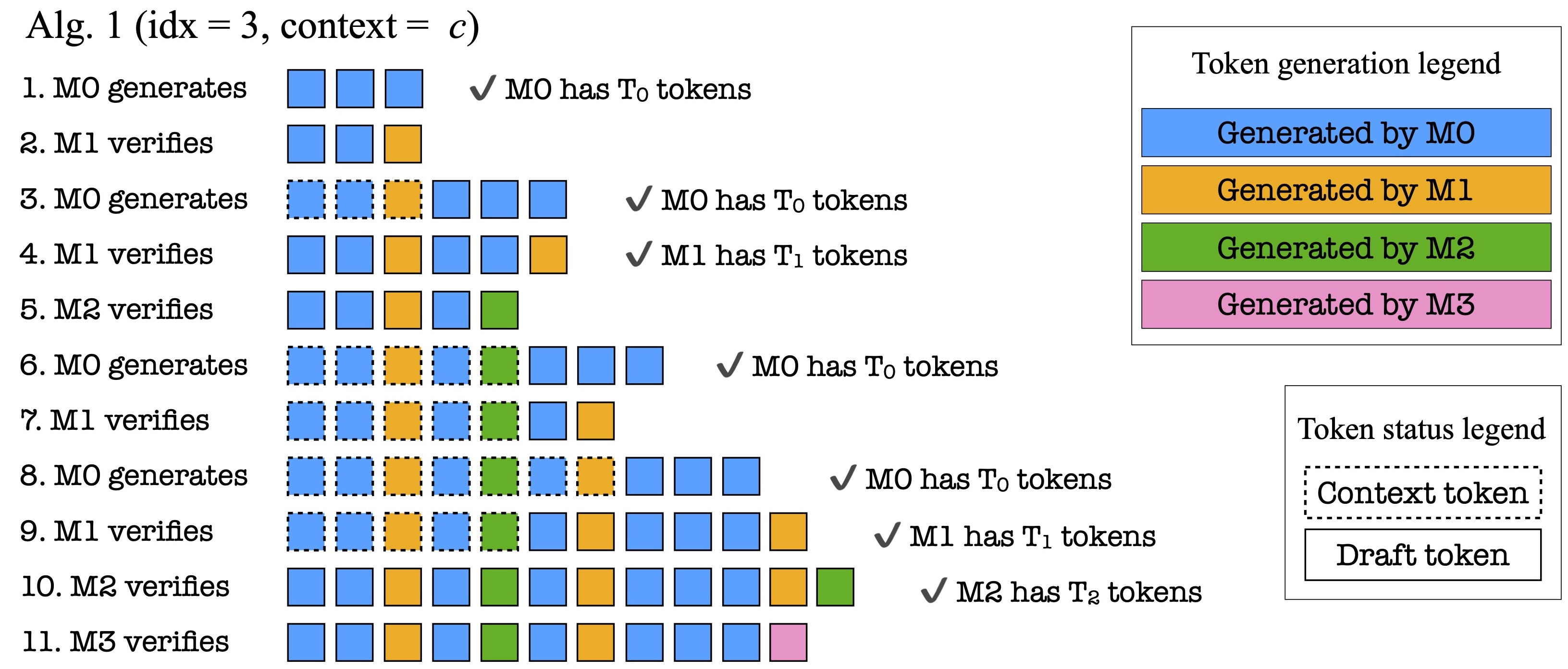} 
    \caption{\looseness=-1 An example stack trace of HSD for $T_0 = 3, T_1 = 6, T_2 = 12$. The color of a token represents which model generated that token. A token can be generated either auto-regressively by the base model $\cM_0$, or by the verification rule which can either replace a token or generates an additional token when all draft tokens are accepted. A token is considered to be part of the context for a certain model if a model above it accepted this token.  }
    \label{fig:wide}
\end{figure*}
We begin with a brief overview of speculative decoding. Given two language models $\cM_q, \cM_p$, the goal is to perform speculative decoding where $\cM_q$ is a small draft model and $\cM_p$ is a large target model. These may be arbitrary models, provided they share the same vocabulary $\mathcal{V}$.

\looseness=-1 For a context $c$, $\cM_q$ has an output distribution over next tokens which is $q_{c} \in [0, 1]^{\abs{\mathcal{V}}}$. That is, $\Pr_{x \sim q_{c}}[x = x_t]$ is the probability that, given context $c$, $\cM_q$ outputs $x_t \in V$ as the next token. Similarly, $\cM_p$ has an output distribution over next tokens which is $p_{c} \in [0, 1]^{\abs{V}}$. Note the output distribution is a function of the context $c$.

The algorithm for speculative sampling is as follows: given $c$ as the context, first sample $x_t \sim q_{c}$. If $q_{c}(x_t) \leq p_{c}(x_t)$, then accept $x_t$. If $q_{c}(x_t) > p_{c}(x_t)$, then accept $x_t$ with probability $p_{c}(x_t)/q_{c}(x_t)$. Otherwise, with probability $1 - p_{c}(x_t)/q_{c}(x_t)$, reject $x_t$ and sample from the distribution defined as follows:
\[
p'_{c}(x) = \frac{\max\{0,  p_{c}(x) - q_{c}(x) \}}{\sum_{x' \in V}\max\{0,  p_{c}(x') - q_{c}(x') \} } \ \forall x.
\]
This rejection sampling technique guarantees that the law of accepted tokens is the same as $p_c$.
The acceptance rate $\alpha_c$ is the the probability that $x_t \sim q_c$ is rejected in the algorithm. It can be derived analytically as follows:
\[
\alpha_c = 1 - \sum_x \abs*{\frac{p_c(x) - q_c(x)}{2}}.
\]

We refer the reader to \citet{leviathan2023fast} for further details and proofs. We also make use of the expected acceptance rate over the input distribution $\cD$, $\alpha = \E_{c \sim \cD}[\alpha_c].$ As with previous literature \citep{leviathan2023fast} we assume for our theoretical purposes that acceptances occur in an independent and identically distributed fashion. Although this is not necessarily the case in practice, we empirically validate that this assumption is not too strong in Section~\ref{sec:experiments} as well as in Appendix~\ref{app:b}.

%% file: hsd.tex
\section{Hierarchical Speculative Decoding}
\label{sec:hsd}
Next, we introduce the Hierarchical Speculative Decoding algorithm. The algorithm leverages several  draft models in order to generate samples from a target model, and can improve upon the latency of using a single draft model.
A key idea in our framework is that models in the hierarchy serve as both drafters and verifiers. Given there are many ways in which one could use a set of models within a hierarchical framework, we begin in Section \ref{subsec:desidarata} with the desiderata which motivate the algorithm. We introduce the algorithm in Section \ref{subsec:hsd_alg} and analyze its latency in Section \ref{subsec:latency_analysis}.  

\subsection{Motivation}
\label{subsec:desidarata}

 We begin by motivating the design of our algorithm. A desired property of an algorithm for speculative decoding with many models is that as many tokens as possible should be processed in parallel. Speculative decoding is successful due to the verifier's ability to verify at least one token in parallel, at a cost similar to generating a single token. This is also efficient in terms of hardware: because there is a significant overhead to loading in the weights of a model onto a device, it is desirable to make the most use of this operation as possible. By parallelizing verification, the same models acts on different tokens simultaneously. As such, we aim to leverage this principle.

Second, we would like for only the smallest model to perform autoregressive generation. This is in order to minimize the initial cost of generating a draft token throughout the algorithm.

Lastly, the algorithm should be principled in the following manner: there should exist configurations in which adding more models to the hierarchy improves the latency from the target model.

We design an algorithm grounded in these desiderata. In a hierarchy of models, we only allow the smallest model to generate tokens autoregressively. After this, all models until the target model function as both drafters and verifiers for the next model.
Prior to verification, we ensure that there are a fixed amount of tokens to be verified in order to maximize parallelism. When a rejection occurs, we supply the verifying model with more draft tokens, rather than allowing it to generate any further tokens on its own or simply passing the remaining tokens to a subsequent model. 

\subsection{Preliminaries}
\begin{algorithm*}[t]
\footnotesize
\caption{Hierarchical Speculative Decoding (HSD)}
\label{alg:seq_generation}
\begin{algorithmic}[1]
\Procedure{HSD}{$\text{idx}, \text{context}$}
    \State \textbf{Input:} Model index $\text{idx} \in [0, K]$, token sequence $\text{context}$
    \State \textbf{Given:} Models $\{\mathcal{M}_0, \ldots, \mathcal{M}_{K}\}$, $T$ values $\{T_0, \ldots, T_{K-1}\}$
    \State $\text{tokens} \gets [\,]$, $\text{probs} \gets [\,]$, $\text{count} \gets 0$

    \If{$\text{idx} = 0$} \Comment{Base case: Smallest model generates autoregressively}
        \While{$\text{count} < T_0$}
            \State $\text{new\_token}, \text{new\_prob} \gets \mathcal{M}_0(\text{context} \mathbin{+} \text{tokens})$
            \State \Call{Append}{$\text{tokens}, \text{new\_token}$}
            \State \Call{Append}{$\text{probs}, \text{new\_prob}$}
            \State $\text{count} \gets \text{count} + 1$
        \EndWhile

    \ElsIf{$\text{idx} = K$} \Comment{Top case: Target model verifies drafts from below}
        \State $\text{draft\_tokens}, \text{draft\_probs} \gets \Call{HSD}{\text{idx} - 1, \text{context}}$
        \State $\text{tokens}, \text{probs} \gets \Call{Verify}{\text{idx}, \text{draft\_tokens}, \text{draft\_probs}, \text{context}}$

    \Else \Comment{Recursive step: Intermediate models verify and extend}
        \While{$\text{count} < T_{\text{idx}}$}
            \State $\text{draft\_tokens}, \text{draft\_probs} \gets \Call{HSD}{\text{idx} - 1, \text{context} \mathbin{+} \text{tokens}}$
            \State $\text{verified\_tokens}, \text{verified\_probs} \gets \Call{Verify}{\text{idx}, \text{draft\_tokens}, \text{draft\_probs}, \text{context} \mathbin{+} \text{tokens}}$
            \State \Call{Extend}{$\text{tokens}, \text{verified\_tokens}$}
            \State \Call{Extend}{$\text{probs}, \text{verified\_probs}$}
            \State $\text{count} \gets \text{count} + \Call{Len}{\text{verified\_tokens}}$
        \EndWhile
    \EndIf
    
    \State \textbf{return} $\text{tokens}, \text{probs}$
\EndProcedure
\end{algorithmic}
\end{algorithm*}
We are given language models $\cM_0, \cM_1, \ldots, \cM_K$, where $\cM_K$ is the target model. All models share the same vocabulary $\mathcal{V}$, but are otherwise arbitrary. For example, models could be early-exits at different stages from the same model \citep{schuster2022confident}. Each model $\cM_i$ has an inference cost  $c_i > 0$ which, for our purpose, is the time to complete a forward pass. The verification cost is similar to the token generation cost. We also assume that the cost of verifying a batch of tokens is the same as one token generation. The acceptance rate between $\cM_i$ and $\cM_j$ is $\alpha_{i, j} \in [0, 1]$, as defined in Section \ref{sec:background}. $\cM_i$ takes as input a context $c \in \{\mathcal{V}\}^L$, where $L > 0$ is the context length. It outputs a tuple $(t, p)$ where $p \in [0, 1]^{\abs{V}}$ is the distribution over the next token and $t \sim p$.


\subsection{Main algorithm}

\label{subsec:hsd_alg}

We introduce HSD in Algorithm \ref{alg:seq_generation}, a recursive procedure in which each model in the hierarchy requests draft tokens from the model below it. Upon receiving these draft tokens, verification is performed. Every model, except the final target model, maintains a buffer of verified tokens that must reach a specified size before returning tokens upstream. Figure~\ref{fig:alg} illustrates an example stack trace.

To generate tokens from the target model $\cM_{K}$, the process begins with the a call to HSD with the initial context and the model index set to $K$. The recursion reaches the base case when the smallest model, $\cM_0$, generates tokens autoregressively.  $\cM_0$ generates $T_0$ tokens sequentially, which it passes to model $\cM_1$ for verification. $\cM_1$ verifies these tokens, and if fewer than $T_1$ tokens have been accepted, $\cM_0$ continues generating batches of $T_0$ tokens for verification by $\cM_1$.  The verification procedure is detailed in Appendix~\ref{app:b}. 



Pseudo-code for the verification function is provided in Appendix~\ref{app:b}, and is the same as that of \citet{leviathan2023fast}. Throughout, we use `$+$' to denote the concatenation of token sequences.

We state the correctness of HSD. The proof is given in Appendix~\ref{app:a}.
\begin{theorem}[Correctness of HSD]
\label{thm:alg_correctness}
    For any set of models $\cM_0, \ldots, \cM_K$, where $\cM_K$ is the target model and any parameters, the output distribution of Algorithm~\ref{alg:seq_generation} follows that of target model $\cM_K$.
\end{theorem}

\subsection{Latency analysis}
\label{subsec:latency_analysis}
In this section, we derive an expression for the latency when all models are included. In order to analyze the latency theoretically, we assume that acceptances occur in an IID fashion, as already stated. This also implies that the number of tokens generated per recursive round is  IID. Empirical results in Section \ref{sec:experiments} show that it is not an unreasonable assumption due to the alignment of the derived expected latency and the true latency.

We define the function $\gamma: [0, 1] \times \Nset \times \Nset \rightarrow \Rset$, which counts the expected number of rounds of drafting and verification between a given pair of models. In particular, if $\cM_j$ requires $T_j$ tokens but receives $T_i$ tokens at a time from $\cM_i$, then $\gamma(\alpha_{i, j}, T_{i}, T_j)$ is the expected number of draft and verification rounds. Each one of these rounds results in a recursive call querying $\cM_i$ for more tokens. In practice, we estimate the value of $\gamma$ empirically. \footnote{Since we receive the tokens in multiples of $T_j$, we may collect more than $T_i$ token. This makes it difficult to give an exact formula for $\gamma$.}

\begin{theorem}
\label{thm:latency1}
For a set of models $\{\cM_i\}_{i \in [K]}$ where the pairwise acceptance rates are $\alpha_{i, j}$ for all $i, j \in [K]$, and parameters  $T = \{T_0, \ldots, T_{K-1}\}$, the expected latency per token of HSD is:
\[ 
\sum_{i = 0}^K c_i \prod_{j = i}^K R(\alpha_{j-1, j}, j),
\]    
where $R: [0,1] \times [K] \rightarrow \Rset$ is defined as:
\[
R(\alpha, n) = \begin{cases} 
\paren*{1 - \alpha}/\paren*{1 - \alpha^{T_{K-1}+1}} &\text{if }  n = K \\
\gamma(\alpha,  T_{n-1}, T_n) &\text{if } 1 \leq n < K \\
T_0 & \text{if } n = 0.
\end{cases}
\]
\end{theorem}
\begin{proof}
We give a proof by induction over the value of $idx$ given to Algorithm \ref{alg:seq_generation}. In the base case, $idx = 0$. The cost of the algorithm in this case is simply $T_0c_0$. 
The inductive hypothesis states that for all $k < K-1$,  the cost of Algorithm \ref{alg:seq_generation} with $idx = k$ is $\sum_{i = 0}^k c_i \prod_{j = 1}^k R(\alpha_{j-1, j}, j)$. Consider now the case where $idx = k+1$. According to the function description, while $T_{k+1}$ tokens have not been accepted, further tokens will be requested from $\cM_{k}$ via function calls of the algorithm with $idx = k$. The expected number of such rounds is $\gamma(\alpha_{k, k+1}, T_k, T_{k+1})$. By the inductive hypothesis, in expectation, each of these rounds takes time $\sum_{i = 0}^k c_i \prod_{j = i}^k R(\alpha_{j-1, j}, j) + c_{k+1}$. The additional cost of $c_{k+1}$ is incurred due to verification. The expected cost at $idx = k+1$ is thus:
\begin{align*}
& R(\alpha_{k, k+1}, k+1) \paren*{\sum_{i = 0}^k c_i \prod_{j = i}^k R(\alpha_{j-1, j}, j) + c_{k+1}} \\
&= \sum_{i = 0}^{k+1} c_i \prod_{j = i}^{k+1} R(\alpha_{j-1, j}, j).
\end{align*}
Hence, the inductive hypothesis holds for all $k < K$. If $idx = K$, we must instead divide by the expected number of generated tokens in order to obtain the latency. This is because $\cM_K$ verifies all tokens from $\cM_{K-1}$ and outputs those which it accepts. As shown in \citet{leviathan2023fast}, the expected number of tokens generated from $\cM_{K}$ is $(1 - \alpha_{K-1, K}^{T_{K-1}+1})/(1 - \alpha_{K-1, K})$.
\end{proof}
\subsection{Motivating example}
Having analyzed the expected latency of HSD, we return to the question: does there exist a configuration of models such that increasing the number of models included in the hierarchy decreases the latency from the target model? We answer this question in the affirmative with an example configuration in Table~\ref{tab:expected_latency_comparison}, and provide details of the configuration used in the Appendix~\ref{app:b}. As this is only one example, we note that it is likely there exist configurations which exhibit even greater speedup from including more models. 


\begin{table}[t]
\centering
\begin{tabular}{ccc}
\toprule
\textbf{\makecell{Number of\\Models}}&\textbf{\makecell{Expected\\Speedup}}&\textbf{\makecell{Expected\\Latency}}\\
\midrule
1&1.0000×&33.00\\
2&2.2971×&14.37\\
3&3.0211×&10.89\\
4&3.0620×&10.64\\
5&3.0829×&10.63\\
6&3.0839×&10.61\\
\bottomrule
\end{tabular}
\caption{An example of the expected speedup as the number of models provided to HSD increases.}
\label{tab:expected_latency_comparison}
\end{table}

%% file: hsd_optimization.tex
\section{Efficient Optimization of Hierarchies}
\label{sec:optimization}
The HSD algorithm is specified by a set of models and parameters. Thus, given a set of $K$ potential draft models from which to choose, there are $O(2^K)$ possible sets of models. A question which arises is, \textit{how do we find the hierarchy with the best latency?} Including all models might not necessarily be the optimal solution: perhaps there is a model which suffers a poor acceptance rate to the subsequent model, or perhaps two models are somewhat redundant. The problem becomes even more challenging when the objective is also to identify the optimal $T$ parameters. 


Finding the optimal hierarchy naturally requires having an estimate of the latency corresponding to each hierarchy. While this could be obtained via simulation, it would be costly and inefficient. In Section~\ref{sec:subset_latency}, we provide the latency analysis for a subset of models. In Section~\ref{subsec:reduction}, we show that, after selecting a maximum value for any parameter in $T$, the optimization can in fact be solved in polynomial time.


\subsection{Latency of a subset of models \label{sec:subset_latency}}
We present Corollary \ref{cor:best_subset}, a natural extension of Theorem \ref{thm:latency1} that is useful for discussing the latency of a subset of models, rather than the entire set of models. The proof follows from that of Theorem~\ref{thm:latency1}.
\begin{corollary}
\label{cor:best_subset}
Given models $\{\cM_i\}_{i=0}^K$, an ordered subset $\sigma \subseteq [K]$ of model indices with $\abs{\sigma} \geq 2$ and final element $K$, and parameters $T = \{T_0, \ldots, T_{\abs{\sigma}-1}\}$, the expected latency of HSD using  $\{\cM_i\}_{i \in \sigma}$ is: 
\begin{align*}
L(\sigma, T) = \sum_{i = 0}^{\abs{\sigma}} c_{\sigma[i]} \prod_{j = i}^{\abs{\sigma}} R_{\sigma,T}(\alpha_{\sigma[j-1], \sigma[j]}, j), 
\end{align*}
where $R_{\sigma, T}: [0, 1] \times [\abs{\sigma}] \rightarrow \Rset$ is defined as:\vspace{-3pt}
\begin{equation*}
R_{\sigma, T}(\alpha, n) = \begin{cases} 
\paren*{1 - \alpha} /\paren*{1 - \alpha^{T_{\abs{\sigma}-1}+1}} &\text{if }  n = \abs{\sigma} \\ 
\gamma(\alpha, T_{n-1},T_n) &\text{if } 1 \leq n < \abs{\sigma} \\
T_0 & \text{if } n = 0.
\end{cases}
\end{equation*}
\end{corollary}
\subsection{Preliminaries}
\begin{definition}[HSD problem]
Given a set of models $\{\cM_0\}_{i = 0}^K$ where $\cM_K$ is the target model, and their pairwise acceptance rates, find the subset and parameters which attain the minimum latency $L^*$:
\[
L^* = \min_{\sigma, T} L(\sigma, T).
\]
\end{definition} 

Assuming a maximum $T$ parameter value, we next show that the HSD problem can be solved via a reduction to the Generalized Shortest Path problem~\citep{oldham2001combinatorial}, which is defined below.
\begin{definition}[Generalized Shortest Path (GSP) Problem]
    Given a directed graph $G = (V, E)$, an edge multiplier $\mu: E \rightarrow \Rset > 0$, an edge cost function $c: E \rightarrow \Rset$, and a source vertex $v \in V$, find the flow function $f: E \rightarrow \Rset \geq 0$ which satisfies:
    \begin{align*}
        \min \quad & \sum_{e \in E} f(e)c(e) \\
        \text{s.t.} \quad & \sum_{(v, w) \in E} f(v, w) - \sum_{(u, v) \in E} \mu(u, v) f(u, v) \\
        & \qquad = \mathbb{I}[v = s], && \forall v \in V \\
        & f(e) \geq 0, && \forall e \in E.
    \end{align*}
\end{definition}
\vspace{-5pt}
GSP describes a problem in which one unit of flow is sent from a designated source vertex. The objective is to find a path which minimizes the cost of sending out this unit of flow, subject to the constraint that the path must be flow-conserving. A key challenge in GSP is that, in addition to edges having a cost $c$, they also have flow multipliers: when flow traverses edge $e$, the flow is multiplied by $\mu(e)$. Given a graph with $n$ vertices and $m$ edges, GSP can be solved in $O(mn^2 \log n)$ time \citep{oldham2001combinatorial}.
We give two definitions to be used in Lemma \ref{lem:oldham}, which motivates our reduction.
\begin{definition}
A \emph{lossy cycle} is a cycle whose product of flow multipliers is strictly less than $1$.
\end{definition}

\begin{definition}
An \emph{augmented path} $s \leadsto v \leadsto w \rightarrow v$\footnote{$s \leadsto v$ denotes some path starting at $s$ and ending at $v$.} is a nonempty path $s \leadsto v \leadsto w$ with an extra edge $w \rightarrow v$ forming
a lossy cycle $v \leadsto w \rightarrow v$. It is a feasible solution to the GSP because the path transports the source’s unit supply to a lossy cycle which “consumes” the flow reaching it.
\end{definition}
\renewcommand{\arraystretch}{1.5}
{
\setlength{\tabcolsep}{2pt} 
\begin{table}[t]
    \centering
    \label{tab:edge_costs} 
    \begin{tabular}{lcc}
        \toprule
        \textbf{Edge $(u, v)$} & \textbf{Multiplier $\mu(u, v)$} & \textbf{Cost $c(u, v)$} \\
        \midrule
        $(\cM_K) , (\cM_i, j)$ &
        $\dfrac{1 - \alpha_{i, K}}{1 - \alpha_{i, K}^{j + 1}}$ &
        $\dfrac{1 - \alpha_{i, K}}{1 - \alpha_{i, K}^{j + 1}} c_K$ \\
        $(\cM_i, j) , (\cM_k, \ell)$ &
        $\gamma(\alpha_{k, i}, \ell, j)$ &
        $\gamma(\alpha_{k, i}, \ell, j)\, c_i$ \\ 
        $(\cM_i, j) , (\cM_i, L)$ &
        $1$ &
        $j\, c_i$ \\ 
        $(\cM_i, L) , (\cM_i, L)$ &
        $\frac{1}{2}$ &
        $0$ \\
        \bottomrule
    \end{tabular}
    \caption{Costs and multipliers for different graph edges.}
    \label{tab:new}
\end{table}
}
\begin{lemma}[\cite{oldham2001combinatorial}]
\label{lem:oldham}
    Solutions to GSP must be augmented paths, or convex combinations of augmented paths with the same cost.
\end{lemma}
Without loss of generality, we assume there is a unique augmented path which is  optimal  because an augmented path with equivalent cost can be obtained from a convex combination of such paths. The path which is the solution is determined by all the edges $e$ for which $f(e) > 0$.

\subsection{Reduction from HSD to GSP}
\label{subsec:reduction}
Consider a set of models $\{\cM_i\}_{i = 0}^{K}$, where $\cM_K$ is the target model. Each model $\cM_i$ has cost $c_i$, and the acceptance rates are $\alpha_{i,j}$, $i,j\in [K]$. We set the maximum value for any of the $T$ parameters to be $\overline{T} \in \Nset$. We create a graph $G$ with the following vertices:
\vspace{-8pt}
\begin{enumerate}
    \item $(\cM_K)$,
    \vspace{-2pt}    
    \item $(\cM_i, j) \in \{\cM_i\}_{i = 0}^{K-1} \times \{1, \ldots, \overline{T}\}$,
    \vspace{-2pt}    
    \item  $(\cM_i, L) \in \{\cM_i\}_{i = 0}^{K-1} \times \{L\}$. 
\end{enumerate}
\vspace{-8pt}    
The first vertex, $(\cM_K)$, is the source vertex and corresponds to the target model. The second set of vertices  correspond to the choices of models and parameter to use in the hierarchy.  The third category of vertices are self-looping vertices representing the smallest model in the hierarchy. The graph $G$ has directed edges:
\vspace{-5pt}
\begin{enumerate}
\item  $(\cM_K) \rightarrow (\cM_i, j) \ \ \forall  i, j$,
\vspace{-2pt}  
\item $(\cM_i, j) \rightarrow (\cM_k, \ell) \ \ \forall i > k, j \geq \ell$,
\vspace{-2pt}  
\item $(\cM_i, j) \rightarrow (\cM_i, L) \ \  \forall i, j$,
\vspace{-2pt}  
\item $(\cM_i, L) \rightarrow (\cM_i, L) \ \ \forall i$.
\end{enumerate}
\vspace{-5pt}  
Having defined the  edges, we define $\mu$ and $c$ over these edges as shown in Table~\ref{tab:new}.  We provide an example visualization of the graph reduction in Appendix \ref{app:c}.

\renewcommand{\arraystretch}{1.0}
\begin{table*}[t]
    \centering
    \begin{tabular}{l l c c c c}
    \toprule
        \textbf{Model} & \textbf{Method} & \textbf{$\sigma$} & $T$  & \textbf{Speedup($\uparrow$)} & \textbf{Seconds per Token($\downarrow$)} \\
        \midrule
        \multirow{3}{*}{\shortstack[c]{LayerSkip2-7B \\ (CNN-DM)}} & HSD & $[7, 9, 32]$ & $[2, 5]$ & \textbf{1.76×} & 0.0102 \\ 
        & Baseline & $[8, 32]$ & $[12]$ & 1.62× & 0.0113 \\
        & Target Model & - &  - & 1.00×  & 0.0182\\
        \midrule
        \multirow{3}{*}{\shortstack[c]{LayerSkip2-13B \\ (CNN-DM)}} & HSD & $[7, 18, 40]$ & $[2, 6]$ & \textbf{1.41×} &0.0162 \\ 
        & Single Draft & $[15, 40]$ & $[12]$ & 1.20×  &  0.0190\\
        & Target Model & - & - & 1.00× & 0.0228 \\
        \midrule
        \multirow{3}{*}{\shortstack[c]{LayerSkip2-70B \\ (CNN-DM)}} & HSD & $[7, 23, 80]$ & $[2, 6]$ & \textbf{1.77×} & 0.0410 \\ 
        & Single Draft & $[19, 80]$ & $[5]$ & 1.58× &  0.0456  \\
        & Target Model & - & - & 1.00× & 0.0723 \\
        \midrule
        \multirow{3}{*}{\shortstack[c]{Gemma2-9B \\ (CNN-DM)}} & HSD & $[0, 2, 42]$ & $[1, 1]$ & \textbf{1.06×} & 0.0407 \\ 
        & Single Draft & $[1, 42]$ & $[2]$ & 1.03× & 0.0418 \\
        & Target Model  & - &  - & 1.00×  & 0.0430\\
        \midrule
        \multirow{3}{*}{\shortstack[c]{Gemma2-9B \\ (XSUM)}} & HSD & $[0, 1, 42]$ & $[1, 2]$ & \textbf{1.15×} & 0.0373 \\ 
        & Single Draft & $[1, 42]$ & $[2]$ & 1.08× &  0.0395  \\
        & Target Model & - & - & 1.00× & 0.0429 \\
        \bottomrule \\
    \end{tabular}
    \caption{Results for the LayerSkip models and Gemma2 models. We compare HSD against the single draft baseline and the autoregressive baseline. }
    \label{tab:results}
\end{table*}


\label{thm:reduction_correctness}
\begin{theorem} In the above reduction, a path is a solution to the GSP instance defined above if and only if the corresponding hierarchy is an optimal solution to original HSD problem. 
\end{theorem}
The proof relies on showing a bijection between augmented paths in the GSP instance and hierarchies with their parameters in the HSD instance. The bijection shows that the cost of a path in the GSP instance is equal to the latency of the corresponding hierarchy with those parameters in the HSD instance. We defer the proof to Appendix \ref{app:a}.

\subsection{Computational complexity}
\begin{theorem}
HSD can be solved in time $O(\overline{T}^4K^4  \log (\overline{T}K))$.
\end{theorem}
\begin{proof}
    In the reduction from HSD to GSP, the number of vertices is $O(\overline{T}K)$ and the number of edges is $O(\overline{T}^2K^2)$. The time to create the GSP instance is thus $O(\overline{T}^2K^2)$. While GSP can be solved using a linear program, significant work~\citep{wayne1999generalized, charnes1966one,wayne1999faster, hochbaum1994simple} has been undertaken to reduce the running time. In particular,~\citet{oldham2001combinatorial} gives a strongly polynomial time algorithm: a GSP instance with $n$ vertices and $m$ edges can be solved in  $O(mn^2 \log n)$. Consequently, HSD can be solved in $O(\overline T^4K^4  \log (\overline{T} K))$. 
\end{proof}
\looseness=-1 This result is a significantly faster than searching over all possibilities, which is prohibitive.

%% file: experiments.tex
\section{Empirical validation}
\label{sec:experiments}

\subsection{Datasets and Draft Models}

Our formalism applies to any set of candidate draft models sharing a vocabulary. For our evaluation, we focus on the case where draft models correspond to layers of the LLM, with a trained output head. We refer to these as early-exit models. Thus, the early-exit model for layer $i$ is the representation at layer $i$, passed through an output head that maps to a distribution over the output vocabulary. Therefore, for a transformer with $L$ layers, we have $L$ candidate draft models, from which we can build a hierarchy. 

We evaluate our method on datasets commonly used for evaluating speculative decoding: CNN-DM~\citep{hermann2015teaching} and XSUM~\citep{xsum-emnlp}. See Appendix~\ref{app:d} for more dataset details. We use two classes of models for evaluation.


{\bf LayerSkip}: The LayerSkip \citep{elhoushi2024layer} class of models have been trained with an early-exit objective. Thus, each of their layers can be used as a draft model. This is done by applying the LM head to any of the layers. We consider the 7B, 13B and 70B versions of these models, which have $32, 40,80$ layers respectively. We use the published checkpoints for each of these models. 

{\bf Gemma2-9B}: The Gemma2-9B model \citep{team2024gemma} has $42$ layers. However, Gemma2-9B was not trained to perform early-exiting like the LayerSkip models, so we undertake additional training; for every layer, we attach a language model head (a linear layer mapping from the embedding dimension to the vocabulary) and train it to match the output distribution of the final layer \citep{hinton2015distilling}. We use a learning rate of $2^{-4}$ with a 5\% linear warmup followed by cosine decay for two epochs. Only the LM head is trained and the backbone remains frozen. We train a variant of this model using the respective training sets for CNN-DM and XSUM, and further finetune the smallest model in the hierarchy to match the intermediate model. 

In both classes of models, we have a candidate draft model ${\cal M}_i$ for each layer $i$. In the Gemma setting, memory overhead grows linearly with the number of models included in the hierarchy. This is because one LM head is required to be stored in memory per model. In our experiments, the model and additional LM heads fit comfortably onto one GPU. Because the LayerSkip model uses the same LM head for each layer, the same overhead does not apply in the LayerSkip setting.


In contrast to standard autoregressive decoding, speculative decoding introduces an additional memory overhead. This is because the output distributions of all draft tokens must be stored. As with speculative decoding, HSD also incurs this overhead.  All experiments are performed using NVIDIA H100 GPUs.

\subsection{Optimization and Results}
In order to find the optimal hierarchy, we first require knowledge of the acceptance rates $\alpha_{i,j}$ for each pair of candidate models. We approximate the rates in an efficient manner by doing a pass over the dataset for a subset of prompts, recording the output distributions from all layers during each forward pass, and computing the empirical acceptance rate using the total variation distance of distributions. This takes about one hour with four GPUs. In fact, we believe this process can be further parallelized significantly, although it was not the focus of this paper. Simultaneously, we record the cost associated with each layer. We use these values to create our GSP instance, and run a GSP solver. 

The GSP solver identifies the optimal hierarchy given the acceptance rates and costs. We consider a sufficiently large maximum value for $T$ to be $15$. The GSP solver runs on a CPU and takes about two hours to run to completion.  We consider the following baselines. 

{\bf Single Draft}: Among all candidate models we take use the one that results in the minimal latency when used as a single draft model as in standard speculative decoding. We use the acceptance rates to identify the optimal two-layer setting as suggested in \citet{leviathan2023fast}.  This baseline checks if using a hierarchy is advantageous over a single draft. 

{\bf Target Model}: We sample autoregressively from the target model, without any speculative decoding in order to demonstrate the speedup with respect to generating directly from the target model.

There are of course many other potential baselines. However, since other methods such as \citet{ankner2024hydra,eagle} work on top of the standard two layer speculative decoding setting, we believe that the hierarchical setting should be extended to such methods in order to provide a valid baseline. Hence, we leave these methods for future work.  

We present our empirical results in Table~\ref{tab:results}. We report the average time per token as well as the speedup over autoregressive decoding, with batch size one. In all cases studied, HSD improves latency over both the single draft baseline as well as autoregressive generation. The speedup with respect to the single draft baseline is as much as 1.17× faster, showing the benefits of using HSD over standard speculative decoding. The results hold across various model sizes, as we experiment with models going from 7B parameters to 70B parameters. The greatest improvement can be seen on the LayerSkip class of models, showing that pretraining with an early-exit loss yields better draft models. 

The results in Table~\ref{tab:results} show that, by spending a few hours of compute once, one can obtain significantly faster inference from the target model than with standard speculative decoding. 

%% file: conclusion.tex
\section{Conclusion}
\label{sec:conclusion}
In this paper, we introduce an algorithm which extends speculative decoding to a more general  setting, involving multiple models of varying cost and accuracy. We show that, given the acceptance rates between models, the optimal hierarchy in Hierarchical Speculative Decoding (HSD) can be found efficiently via a polynomial-time algorithm. Empirically, we confirm that our theoretical insights hold in practice and yield an improvement upon standard speculative decoding.

Future work could explore how to integrate HSD with other speculative decoding techniques for the single-draft setting, such as \citep{gloeckle2024better, cai2024medusa, miao2024specinfer}. As an extension, it would be valuable to study how to adapt HSD to the online setting where the hierarchy is chosen as a function of the prompt. Additionally, while our focus is on language models, the HSD framework can be viewed more broadly as a form of rejection sampling, and may apply to other domains such as random walks with heavy-tailed transitions in graphs.

%% file: app_a.tex
\section{Proofs}
\label{app:a}

\subsection{Proof of Theorem 3.1}
\begin{proof}
    We give a proof by induction. We show that for all $i \in [K]$, Algorithm \ref{alg:seq_generation} returns the output distribution of $\cM_i$.
    
    The base case in the context of this proof is when $idx = 1$, and $\cM_1$ verifies the draft tokens obtained from $\cM_0$. Due to the verification rule, all verified tokens in $out$ follow the distribution of $\cM_1$. Furthermore, by outputting the token probabilities obtained directly $\cM_1$, we also have the correct probabilities over next-tokens. 
    
    Suppose the inductive step holds for all values of $idx \leq k$. Next, we consider the case in which $idx = k+1$. Then, the function call to Algorithm \ref{alg:seq_generation} with $idx = k$ correctly returns output tokens and their distributions according to the true distribution of $\cM_k$. Hence, when $\cM_{k+1}$ performs verification of tokens and replaces the token probabilities with its own, we obtain an output token distribution according to that of $\cM_{k+1}$. 

    Hence, it follows that for $idx = K$, the output distribution over tokens is guaranteed to follow the distribution of $\cM_K$. 
\end{proof}

\subsection{Proof of Theorem 4.3}
\begin{proof}
The proof relies on showing a bijection between augmented paths in the GSP instance and hierarchies with parameters in the HSD instance. Recall that all optimal solutions to GSP are augmented paths.  Hence, upon solving GSP and decoding the solution into a path, we obtain a set of vertices along a simple path terminating at a lossy cycle. Define $P$ as the set of all paths in $G$ which start at $s$ and terminate at a lossy cycle. By construction, all lossy cycles have zero cost and the objective can be re-written as a minimization over paths that terminate at a loop vertex:
\begin{align*}
    \min \sum_{e \in E} f(e)c(e) &= \min_{p = (e_1, \ldots, e_{\abs{p}}) \in P} \sum_{i = 1}^{\abs{p}} f(e_i)c(e_i) \\
    &= \min_{p = (e_1, \ldots, e_{\abs{p}}) \in P} \sum_{i = 1}^{\abs{p}} \paren*{\prod_{j = 1}^{i-1} \mu(e_j) }c(e_i).
\end{align*}
Any fixed augmented path $p$ in the graph is of the following form  ($\ell \geq 0$):
\[p = (\cM_K) \rightarrow (\cM_{p_1}, t_{p_1}) \rightarrow \cdots \rightarrow(\cM_{p_\ell}, t_{p_\ell}) \rightarrow (\cM_{p_{\ell+1}}, t_{p_{\ell+1}}) \rightarrow (\cM_{p_{\ell+1}}, L) \circlearrowleft.
\]
The corresponding set of models in HSD is $\cM_{p_{\ell+1}}, \cM_{p_{\ell}}, \ldots, \cM_{p_1}, \cM_K$ and the corresponding $T$ parameters are $j_{p_{\ell+1}}, j_{p_\ell}, \ldots, j_{p_1}$. Denote $\sigma = \{p_{\ell+1}, p_\ell, \ldots, p_1, K\}$. 
Substituting in the values from $\mu$ and $c$, the cost of this path in the GSP instance is:
\begin{align*}
& \frac{(1 - \alpha_{p_1, K})}{(1 - \alpha_{p_1, K}^{t_{p_1}})} c_K + \frac{(1 - \alpha_{p_1, K})}{(1 - \alpha_{p_1, K}^{t_{p_1}})} \gamma(\alpha_{p_2, p_1}, t_{p_2}, t_{p_1}) c_{p_1} + \frac{(1 - \alpha_{p_1, K})}{(1 - \alpha_{p_1, K}^{t_{p_1}})} \gamma(\alpha_{p_2, p_1}, t_{p_2}, t_{p_1})\gamma(\alpha_{p_3, p_2}, t_{p_3}, t_{p_2})c_{p_2} +   \cdots   \\
&= \sum_{i = 0}^{\abs{\sigma}} c_{\sigma[i]} \prod_{j = i}^{\abs{\sigma}} 
 R_{\sigma, T}(\alpha_{\sigma[j-1], \sigma[j]}, j) = L(\sigma, T).
\end{align*}
Hence, the cost of a path in the GSP reduction is equal to the latency of the hierarchy which it specifies. By construction, every possible hierarchy is encoded as an augmented path in the GSP reduction; this is because any model can terminate the augmented path due to its designated lossy cycle vertex. Furthermore, every augmented path in the graph corresponds to exactly one subset $\sigma$ and $T$ parameters. Thus, there exists a bijection between augmented paths in the GSP instance and hierarchies in the HSD instance. Because the cost of a path in GSP exactly corresponds to the latency of that hierarchy, a path is a solution to the GSP instance if and only if the corresponding hierarchy is a solution to HSD.
\end{proof}

%% file: app_b.tex
\newpage
\section{Sections 2 and 3 Details}
\label{app:b}

\subsection{Verification algorithm}
First, we provide the algorithm description for verification.

\begin{algorithm}[ht]
\caption{Token Verification and Correction}
\label{alg:verification_revised}
\begin{algorithmic}[1]
\Procedure{Verify}{$\text{idx}, \text{draft\_tokens}, \text{draft\_probs}, \text{context}$}
    \State $t \gets \Call{Len}{\text{draft\_tokens}}$
    \State Let $\text{draft\_tokens} = (x_1, \ldots, x_t)$ and $\text{draft\_probs} = (q_1, \ldots, q_t)$
    \Comment{Run verifier $\mathcal{M}_{\text{idx}}$ in parallel on all prefixes to get true distributions}
    \State $p_1, \ldots, p_{t+1} \gets \mathcal{M}_{\text{idx}}(\text{context}), \ldots, \mathcal{M}_{\text{idx}}(\text{context} + x_1\ldots x_t)$
    \State $n \gets t$ \Comment{Initialize number of accepted tokens to the maximum}
    \For{$i = 1 \to t$}
        \State Sample $r \sim U(0, 1)$
        \If{$r > \frac{p_i(x_i)}{q_i(x_i)}$} \Comment{Rejection sampling condition}
            \State $n \gets i - 1$ \Comment{The first $n$ tokens are accepted}
            \State \textbf{break} \Comment{Exit the loop}
        \EndIf
    \EndFor
    
    \State $\text{accepted\_tokens} \gets (x_1, \ldots, x_n)$
    \State $\text{final\_dist} \gets p_{n+1}$ \Comment{Get distribution for the token after the accepted sequence}
    
    \If{$n < t$} \Comment{If a token was rejected, modify the distribution}
        \State $\text{final\_dist}(x) \gets \Call{Normalize}{\max\{0, p_{n+1}(x) - q_{n+1}(x)\}}$ for all $x$
    \EndIf
    
    \State Sample $m \sim \text{final\_dist}$ \Comment{Sample a corrected token from the final distribution}
    \State $\text{output\_tokens} \gets \text{accepted\_tokens} + [m]$
    \State $\text{output\_probs} \gets (p_1, \ldots, p_n, p_{n+1})$
    
    \State \textbf{return} $\text{output\_tokens}, \text{output\_probs}$
\EndProcedure
\end{algorithmic}
\end{algorithm}

This verification algorithm is exactly the same as that proposed in \citep{leviathan2023fast}.

\subsection{Examining the assumptions of HSD}
We conduct an ablation study to evaluate the impact of two simplifying assumptions made in our theoretical analysis: (1) that acceptance rates are IID, and (2) that generation and verification costs remain constant throughout inference. We use a four-layer hierarchy with Gemma2 9B to introduce more variability than the settings in our main results. 

In order to assess the the validity of the first assumption, we simulate IID acceptance rates. To that end, we replace the verification rule in Algorithm~\ref{alg:verification_revised} with a biased coin toss for each token. The probability of acceptance is set to the empirical average rate. In order to assess the validity of the second assumption, we simulate a constant cost. We substitute the measured wall-clock time at each step with a fixed, artificial cost, and the total latency is the sum of these costs. Thus, we measure latency across four settings and report the results in Table~\ref{tab:latency_comparison}.

\begin{description}
    \item[Real Acceptance / Real Cost] The standard setting, which uses the true acceptances from Algorithm~\ref{alg:verification_revised} and measures actual wall-clock time.
    
    \item[IID Acceptance / Real Cost] We use simulated acceptances but measure actual wall-clock time.
    \item[Real Acceptance / Artificial Cost] We use the true acceptances but measure a fixed, artificial cost per step.
    \item[IID Acceptance / Artificial Cost] This represents the fully simplified model, using both simulated acceptances and fixed costs.
\end{description}

\begin{table}[htbp]
\centering
\begin{tabular}{llr}
\toprule
\textbf{Acceptance Rate} & \textbf{Cost Type} & \textbf{Latency} \\
\midrule
Real & Real & 0.0438226 \\
IID & Real & 0.0437504 \\
Real & Artificial & 0.0439264 \\
IID & Artificial & 0.0438760 \\
\bottomrule
\end{tabular}
\caption{Comparison of latency under different conditions.}
\label{tab:latency_comparison}
\end{table}

As shown, there is very little variability between all of these settings. We ran many  experiments of this nature in order to both validate our assumptions and also verify that our algorithm was indeed running correctly. Hence, we conclude that the assumptions made by our theoretical work are not too strong to capture the empirical aspects.

\subsection{HSD Example}
We expand further on the example provided in Table \ref{tab:expected_latency_comparison}. This example was constructed manually. We constructed this example in order to convey a setting in which adding more models improves the latency of HSD. In our example, we add one more model at a time by adding a new smallest model. Then, we solve for the optimal hierarchy. In our example, every time a new model is added as an option, it is optimal to use it in HSD.

We note the acceptance rate matrix must follow a certain structure. This is because acceptance rates are obtained via the TV distance of distributions, a distance metric that respects the triangle inequality. This implies that for any $i \neq j \neq k$, the following must hold:
\[
\alpha_{i, j} + \alpha_{j, k} \leq \alpha_{i, k} + 1.
\]
We create an acceptance rate matrix as shown in Table~\ref{tab:clean_matrix_2}, where the acceptance rate from $\cM_i$ to model $\cM_j$ is in the $i$'th row and $j$'th column.
\begin{table}[htbp]
\centering
\begin{tabular}{c S[table-format=1.3]
                S[table-format=1.3]
                S[table-format=1.3]
                S[table-format=1.3]
                S[table-format=1.3]}
\toprule
 & {\textbf{2}} & {\textbf{3}} & {\textbf{4}} & {\textbf{5}} & {\textbf{6}} \\
\midrule
\textbf{1} & 0.750 & 0.500 & 0.250 & 0.000 & 0.000 \\
\textbf{2} & {--}  & 0.750 & 0.500 & 0.250 & 0.050 \\
\textbf{3} & {--}  & {--}  & 0.750 & 0.500 & 0.300 \\
\textbf{4} & {--}  & {--}  & {--}  & 0.750 & 0.550 \\
\textbf{5} & {--}  & {--}  & {--}  & {--}  & 0.800 \\
\bottomrule
\end{tabular}
\caption{First example acceptance rate matrix.}
\label{tab:clean_matrix_2}
\end{table}

We use the following costs: 
$c_1 = 0.00001, c_2 = 0.003, c_3 = 0.01, c_4 = 0.25, c_5 = 4, c_6 = 33$. While increasing the number of available models, we run the GSP solver to identify the optimal hierarchy to provide to HSD, and compute the expected latency.

We can instantiate many other such examples simply by changing the costs and acceptance rates. Suppose we let the costs be $c_1 = 0.00005,c_2= 0.0002,c_3= 0.05,c_4= 2.0, c_5 =8.0, c_6=33.0$ and let  the acceptance rate matrix be as in Table~\ref{tab:clean_matrix}.
\begin{table}[htbp]
\centering
\begin{tabular}{c S[table-format=1.3]
                S[table-format=1.3]
                S[table-format=1.3]
                S[table-format=1.3]
                S[table-format=1.3]}
\toprule
 & {\textbf{2}} & {\textbf{3}} & {\textbf{4}} & {\textbf{5}} & {\textbf{6}} \\
\midrule
\textbf{1} & 0.525 & 0.125 & 0.000 & 0.000 & 0.000 \\
\textbf{2} & {--}  & 0.600 & 0.275 & 0.025 & 0.000 \\
\textbf{3} & {--}  & {--}  & 0.675 & 0.425 & 0.225 \\
\textbf{4} & {--}  & {--}  & {--}  & 0.750 & 0.550 \\
\textbf{5} & {--}  & {--}  & {--}  & {--}  & 0.800 \\
\bottomrule
\end{tabular}
\caption{Second example acceptance rate matrix.}
\label{tab:clean_matrix}
\end{table}
Then, adding more models yields speedup as shown in Table~\ref{tab:expected_latency_comparison2}.
\begin{table}[htbp]
\centering
\begin{tabular}{ccc}
\toprule
\textbf{\makecell{Number of\\Models}}&\textbf{\makecell{Expected\\Speedup}}&\textbf{\makecell{Expected\\Latency}}\\
\midrule
1&1.0000×&33.00\\
2&1.7090×&19.31\\
3&2.1366×&15.45\\
4&2.2587×&14.61\\
5&2.2817×&14.46\\
6&2.2910×&14.40\\
\bottomrule
\end{tabular}
\caption{A second example of the expected speedup as the number of models provided to HSD increases.}
\label{tab:expected_latency_comparison2}
\end{table}
\subsection{Analysis of optimal HSD configurations}
Due to the introduction of a new optimization problem for each hierarchy, it is difficult to straightforwardly quantify when introducing a new model would lower the latency. However, we conduct an experiment to explore this. In the experiment, we fix a target model A with cost 1024, and a draft model B with cost 256. We fix the acceptance rate between the two to be 50\%. Then, we introduce a third model, C, where we vary both its cost and its acceptance rate to B to identify settings in which it is optimal to use the hierarchy A-B-C. We use a lower bound to determine the acceptance rate from C to A. For each configuration, we run our optimization algorithm to identify the optimal hierarchy to minimize latency.  

We present our findings in Table~\ref{tab:speedup_hierarchy_right}. For each choice of cost for model C and acceptance rate from model C to B, we solve for the optimal latency. We color-code the cell based on which hierarchy achieves this latency.

As we can see, as the cost of C increases, it is less appealing to use it unless it also has a strong acceptance rate to B. When C has both a low cost and high acceptance rate to A, it eventually becomes optimal only to use model C. In between these scenarios, we see numerous instances in which the complete hierarchy A-B-C is the optimal one to use.

\begin{table}[ht]
\centering
\begin{tabular}{@{}rllllllll@{}}
\toprule
\textbf{Acceptance Rate} & \multicolumn{8}{c}{\textbf{Cost of Model C}} \\
\cmidrule(l){2-9}
\textbf{(C to B)} & 1.0 & 2.0 & 4.0 & 8.0 & 16.0 & 32.0 & 64.0 & 128.0 \\
\midrule
0.0 & \cellcolor{colorABOptimal}1.20 & \cellcolor{colorABOptimal}1.20 & \cellcolor{colorABOptimal}1.20 & \cellcolor{colorABOptimal}1.20 & \cellcolor{colorABOptimal}1.20 & \cellcolor{colorABOptimal}1.20 & \cellcolor{colorABOptimal}1.20 & \cellcolor{colorABOptimal}1.20 \\
0.1 & \cellcolor{colorABOptimal}1.20 & \cellcolor{colorABOptimal}1.20 & \cellcolor{colorABOptimal}1.20 & \cellcolor{colorABOptimal}1.20 & \cellcolor{colorABOptimal}1.20 & \cellcolor{colorABOptimal}1.20 & \cellcolor{colorABOptimal}1.20 & \cellcolor{colorABOptimal}1.20 \\
0.2 & \cellcolor{colorABCOptimal}1.21 & \cellcolor{colorABCOptimal}1.21 & \cellcolor{colorABCOptimal}1.20 & \cellcolor{colorABOptimal}1.20 & \cellcolor{colorABOptimal}1.20 & \cellcolor{colorABOptimal}1.20 & \cellcolor{colorABOptimal}1.20 & \cellcolor{colorABOptimal}1.20 \\
0.3 & \cellcolor{colorABCOptimal}1.23 & \cellcolor{colorABCOptimal}1.23 & \cellcolor{colorABCOptimal}1.22 & \cellcolor{colorABCOptimal}1.22 & \cellcolor{colorABCOptimal}1.21 & \cellcolor{colorABOptimal}1.20 & \cellcolor{colorABOptimal}1.20 & \cellcolor{colorABOptimal}1.20 \\
0.4 & \cellcolor{colorABCOptimal}1.25 & \cellcolor{colorABCOptimal}1.25 & \cellcolor{colorABCOptimal}1.24 & \cellcolor{colorABCOptimal}1.24 & \cellcolor{colorABCOptimal}1.23 & \cellcolor{colorABCOptimal}1.21 & \cellcolor{colorABOptimal}1.20 & \cellcolor{colorABOptimal}1.20 \\
0.5 & \cellcolor{colorABCOptimal}1.27 & \cellcolor{colorABCOptimal}1.27 & \cellcolor{colorABCOptimal}1.27 & \cellcolor{colorABCOptimal}1.26 & \cellcolor{colorABCOptimal}1.25 & \cellcolor{colorABCOptimal}1.23 & \cellcolor{colorABOptimal}1.20 & \cellcolor{colorABOptimal}1.20 \\
0.6 & \cellcolor{colorABCOptimal}1.30 & \cellcolor{colorABCOptimal}1.29 & \cellcolor{colorABCOptimal}1.29 & \cellcolor{colorABCOptimal}1.28 & \cellcolor{colorABCOptimal}1.27 & \cellcolor{colorABCOptimal}1.25 & \cellcolor{colorABCOptimal}1.22 & \cellcolor{colorABOptimal}1.20 \\
0.7 & \cellcolor{colorABCOptimal}1.34 & \cellcolor{colorABCOptimal}1.33 & \cellcolor{colorABCOptimal}1.33 & \cellcolor{colorABCOptimal}1.32 & \cellcolor{colorABCOptimal}1.30 & \cellcolor{colorABCOptimal}1.28 & \cellcolor{colorABCOptimal}1.24 & \cellcolor{colorABOptimal}1.20 \\
0.8 & \cellcolor{colorACOptimal}1.42 & \cellcolor{colorACOptimal}1.41 & \cellcolor{colorACOptimal}1.40 & \cellcolor{colorACOptimal}1.38 & \cellcolor{colorACOptimal}1.35 & \cellcolor{colorACOptimal}1.31 & \cellcolor{colorABCOptimal}1.27 & \cellcolor{colorABCOptimal}1.21 \\
0.9 & \cellcolor{colorACOptimal}1.65 & \cellcolor{colorACOptimal}1.64 & \cellcolor{colorACOptimal}1.63 & \cellcolor{colorACOptimal}1.60 & \cellcolor{colorACOptimal}1.55 & \cellcolor{colorACOptimal}1.48 & \cellcolor{colorACOptimal}1.39 & \cellcolor{colorACOptimal}1.25 \\
\bottomrule
\end{tabular}
\vspace{0.5em} 

\small{%
{\color{colorABCOptimal}\rule{0.8em}{0.8em}}~
 Hierarchy A-B-C is optimal. \quad
{\color{colorABOptimal}\rule{0.8em}{0.8em}}~%
 Hierarchy A-B is optimal. \quad
{\color{colorACOptimal}\rule{0.8em}{0.8em}}~%
 Hierarchy A-C is optimal.
}

\caption{Speedup from optimal hierarchy across various parameters.}
\label{tab:speedup_hierarchy_right}
\end{table}

\newpage

%% file: app_c.tex
\section{Reduction from HSD to GSP}
\label{app:c}
In the following example, we draw the graph of the reduction for when $K = 3$ and $\bar{T} = 3$. The source vertex is $(\cM_3)$ and the functions $\mu$ and $c$ are as defined in Section 4. By finding the cheapest flow-conserving path which takes one unit of flow out of the source vertex, we also find the optimal speculative decoding hierarchy.
\begin{center}
    
\begin{tikzpicture}
\usetikzlibrary{positioning, calc, backgrounds, arrows.meta} 

\definecolor{rootColor}{HTML}{0D47A1}  
\definecolor{baseColor}{HTML}{D84315}  
\definecolor{toLColor}{HTML}{00695C}   
\definecolor{loopColor}{HTML}{546E7A}  
\definecolor{nodeFill}{HTML}{E3F2FD}   
\definecolor{nodeBorder}{HTML}{546E7A} 

\def\xsep{5.0}   
\def\ysep{2.8}   
\def\yGap{4.5}   

\tikzset{
  nodeStyle/.style = {
    draw=nodeBorder, circle, fill=nodeFill,
    minimum size=1.2cm, inner sep=0pt,
    thick
  },
  edgeStyle/.style = {->, opacity=.8, thick},
  rootEdge/.style  = {edgeStyle, rootColor, very thick},
  baseEdge/.style  = {edgeStyle, baseColor},
  toLEdge/.style   = {edgeStyle, toLColor},
  loopEdge/.style  = {edgeStyle, loopColor},
}

\node[nodeStyle] (M3) at ({1*\xsep}, 0)
  {$\mathsf{(\mathcal{M}_3)}$};

\foreach \i in {2,1,0}{
  \foreach \j in {1,2,3}{
    \node[nodeStyle] (M\i-\j)
      at ({(\j-1)*\xsep}, { -2.5 - (2-\i)*\ysep })
      {$\mathsf{(\mathcal{M}_{\i},\,\j)}$};
  }
}

\foreach \i in {2,1,0}{
  \node[nodeStyle] (M\i-L)
    at ({(2-\i)*\xsep}, { -2.5 - 2*\ysep - \yGap })
    {$\mathsf{(\mathcal{M}_{\i},\,L)}$};
}

\begin{scope}[on background layer]
    \foreach \i in {2,1,0}{
      \foreach \j / \out / \inn in {1/-135/45, 2/-90/90, 3/-45/135}{
        \draw[rootEdge, looseness=1.2]
          (M3) to[out=\out, in=\inn] (M\i-\j);
      }
    }

    \foreach \i/\iminus in {2/1,1/0}
      \foreach \k in {0,...,\iminus}
        \foreach \j in {1,2,3}
          \foreach \l in {1,...,\j}{
            \draw[baseEdge] (M\i-\j) to[bend left=15] (M\k-\l);
    }

    \foreach \i in {2,1,0}{
      \foreach \j in {1,2,3}{
        \draw[toLEdge] (M\i-\j) to[out=-90, in=90, looseness=0.6] (M\i-L);
      }
    }

    \foreach \i in {2,1,0}{
      \draw[loopEdge, looseness=6, out=-45, in=-135] (M\i-L) to (M\i-L);
    }
\end{scope}

\end{tikzpicture}
\end{center}

%% file: app_d.tex
\newpage
\section{Dataset Details}
\label{app:d}

We provide details for the XSUM and CNN-DM datasets. 

\begin{table}[ht]
  \centering
  \small
  \begin{tabular}{@{}lll@{}}
    \toprule
    \textbf{Dataset} & \textbf{Domain \& Source} & \textbf{Split Sizes} \\
    & & \textit{train / val / test} \\
    \midrule
    XSum & BBC News articles & 204{,}045 / 11{,}332 / 11{,}334 \\
    CNN/DailyMail & CNN \& Daily Mail news stories & 287{,}226 / 13{,}368 / 11{,}490 \\
    \bottomrule
  \end{tabular}
  \vspace{3pt}
  \caption{Dataset details for XSUM and CNN-DM.}
\end{table}

%% file: main.bbl
\begin{thebibliography}{36}
\providecommand{\natexlab}[1]{#1}
\providecommand{\url}[1]{\texttt{#1}}
\expandafter\ifx\csname urlstyle\endcsname\relax
  \providecommand{\doi}[1]{doi: #1}\else
  \providecommand{\doi}{doi: \begingroup \urlstyle{rm}\Url}\fi

\bibitem[Aishwarya et~al.(2024)Aishwarya, Nair, Samaga, Boyd, Kumar, Jain, and Netrapalli]{tandem}
P.~S. Aishwarya, P.~A. Nair, Y.~Samaga, T.~Boyd, S.~Kumar, P.~Jain, and P.~Netrapalli.
\newblock Tandem transformers for inference efficient llms.
\newblock In \emph{Proceedings of the 41st International Conference on Machine Learning}, ICML'24. JMLR.org, 2024.

\bibitem[Ankner et~al.(2024)Ankner, Parthasarathy, Nrusimha, Rinard, Ragan-Kelley, and Brandon]{ankner2024hydra}
Z.~Ankner, R.~Parthasarathy, A.~Nrusimha, C.~Rinard, J.~Ragan-Kelley, and W.~Brandon.
\newblock Hydra: Sequentially-dependent draft heads for medusa decoding.
\newblock \emph{arXiv preprint arXiv:2402.05109}, 2024.

\bibitem[Bae et~al.(2025)Bae, Fisch, Harutyunyan, Ji, Kim, and Schuster]{bae2024relaxed}
S.~Bae, A.~Fisch, H.~Harutyunyan, Z.~Ji, S.~Kim, and T.~Schuster.
\newblock Relaxed recursive transformers: Effective parameter sharing with layer-wise lora.
\newblock In \emph{The Thirteenth International Conference on Learning Representations}, 2025.
\newblock URL \url{https://openreview.net/forum?id=WwpYSOkkCt}.

\bibitem[Cai et~al.(2024)Cai, Li, Geng, Peng, Lee, Chen, and Dao]{cai2024medusa}
T.~Cai, Y.~Li, Z.~Geng, H.~Peng, J.~D. Lee, D.~Chen, and T.~Dao.
\newblock Medusa: Simple llm inference acceleration framework with multiple decoding heads.
\newblock \emph{arXiv preprint arXiv:2401.10774}, 2024.

\bibitem[Charnes and Raike(1966)]{charnes1966one}
A.~Charnes and W.~M. Raike.
\newblock One-pass algorithms for some generalized network problems.
\newblock \emph{Operations Research}, 14\penalty0 (5):\penalty0 914--924, 1966.

\bibitem[Chen et~al.(2023)Chen, Borgeaud, Irving, Lespiau, Sifre, and Jumper]{chen2023accelerating}
C.~Chen, S.~Borgeaud, G.~Irving, J.-B. Lespiau, L.~Sifre, and J.~Jumper.
\newblock Accelerating large language model decoding with speculative sampling.
\newblock \emph{arXiv preprint arXiv:2302.01318}, 2023.

\bibitem[Deng and Rush(2020)]{deng2020cascaded}
Y.~Deng and A.~Rush.
\newblock Cascaded text generation with markov transformers.
\newblock \emph{Advances in Neural Information Processing Systems}, 33:\penalty0 170--181, 2020.

\bibitem[Dohan et~al.(2022)Dohan, Xu, Lewkowycz, Austin, Bieber, Lopes, Wu, Michalewski, Saurous, Sohl-dickstein, Murphy, and Sutton]{dohan2022language}
D.~Dohan, W.~Xu, A.~Lewkowycz, J.~Austin, D.~Bieber, R.~G. Lopes, Y.~Wu, H.~Michalewski, R.~A. Saurous, J.~Sohl-dickstein, K.~Murphy, and C.~Sutton.
\newblock Language model cascades, 2022.

\bibitem[Elbayad et~al.(2020)Elbayad, Gu, Grave, and Auli]{Elbayad2020Depth}
M.~Elbayad, J.~Gu, E.~Grave, and M.~Auli.
\newblock Depth-adaptive transformer.
\newblock In \emph{International Conference on Learning Representations}, 2020.
\newblock URL \url{https://openreview.net/forum?id=SJg7KhVKPH}.

\bibitem[Elhoushi et~al.(2024)Elhoushi, Shrivastava, Liskovich, Hosmer, Wasti, Lai, Mahmoud, Acun, Agarwal, Roman, Aly, Chen, and Wu]{elhoushi2024layer}
M.~Elhoushi, A.~Shrivastava, D.~Liskovich, B.~Hosmer, B.~Wasti, L.~Lai, A.~Mahmoud, B.~Acun, S.~Agarwal, A.~Roman, A.~Aly, B.~Chen, and C.-J. Wu.
\newblock {L}ayer{S}kip: Enabling early exit inference and self-speculative decoding.
\newblock In L.-W. Ku, A.~Martins, and V.~Srikumar, editors, \emph{Proceedings of the 62nd Annual Meeting of the Association for Computational Linguistics (Volume 1: Long Papers)}, pages 12622--12642, Bangkok, Thailand, Aug. 2024. Association for Computational Linguistics.
\newblock \doi{10.18653/v1/2024.acl-long.681}.
\newblock URL \url{https://aclanthology.org/2024.acl-long.681/}.

\bibitem[Gloeckle et~al.(2024{\natexlab{a}})Gloeckle, Idrissi, Rozi{\`e}re, Lopez-Paz, and Synnaeve]{gloeckle2024better}
F.~Gloeckle, B.~Y. Idrissi, B.~Rozi{\`e}re, D.~Lopez-Paz, and G.~Synnaeve.
\newblock Better \& faster large language models via multi-token prediction.
\newblock \emph{arXiv preprint arXiv:2404.19737}, 2024{\natexlab{a}}.

\bibitem[Gloeckle et~al.(2024{\natexlab{b}})Gloeckle, Idrissi, Rozi\`{e}re, Lopez-Paz, and Synnaeve]{mtp}
F.~Gloeckle, B.~Y. Idrissi, B.~Rozi\`{e}re, D.~Lopez-Paz, and G.~Synnaeve.
\newblock Better \& faster large language models via multi-token prediction.
\newblock In \emph{Proceedings of the 41st International Conference on Machine Learning}, ICML'24. JMLR.org, 2024{\natexlab{b}}.

\bibitem[Gupta et~al.(2024)Gupta, Narasimhan, Jitkrittum, Rawat, Menon, and Kumar]{gupta2024language}
N.~Gupta, H.~Narasimhan, W.~Jitkrittum, A.~S. Rawat, A.~K. Menon, and S.~Kumar.
\newblock Language model cascades: Token-level uncertainty and beyond, 2024.

\bibitem[Hermann et~al.(2015)Hermann, Kocisky, Grefenstette, Espeholt, Kay, Suleyman, and Blunsom]{hermann2015teaching}
K.~M. Hermann, T.~Kocisky, E.~Grefenstette, L.~Espeholt, W.~Kay, M.~Suleyman, and P.~Blunsom.
\newblock Teaching machines to read and comprehend.
\newblock \emph{Advances in neural information processing systems}, 28, 2015.

\bibitem[Hinton et~al.(2015)Hinton, Vinyals, and Dean]{hinton2015distilling}
G.~Hinton, O.~Vinyals, and J.~Dean.
\newblock Distilling the knowledge in a neural network.
\newblock \emph{arXiv preprint arXiv:1503.02531}, 2015.

\bibitem[Hochbaum and Naor(1994)]{hochbaum1994simple}
D.~S. Hochbaum and J.~Naor.
\newblock Simple and fast algorithms for linear and integer programs with two variables per inequality.
\newblock \emph{SIAM Journal on Computing}, 23\penalty0 (6):\penalty0 1179--1192, 1994.

\bibitem[Hooper et~al.(2023)Hooper, Kim, Mohammadzadeh, Genc, Keutzer, Gholami, and Shao]{hooper2023speed}
C.~Hooper, S.~Kim, H.~Mohammadzadeh, H.~Genc, K.~Keutzer, A.~Gholami, and S.~Shao.
\newblock Speed: Speculative pipelined execution for efficient decoding.
\newblock \emph{arXiv preprint arXiv:2310.12072}, 2023.

\bibitem[Kim et~al.(2023)Kim, Mangalam, Moon, Malik, Mahoney, Gholami, and Keutzer]{NEURIPS2023_7b97adea}
S.~Kim, K.~Mangalam, S.~Moon, J.~Malik, M.~W. Mahoney, A.~Gholami, and K.~Keutzer.
\newblock Speculative decoding with big little decoder.
\newblock In A.~Oh, T.~Naumann, A.~Globerson, K.~Saenko, M.~Hardt, and S.~Levine, editors, \emph{Advances in Neural Information Processing Systems}, volume~36, pages 39236--39256. Curran Associates, Inc., 2023.
\newblock URL \url{https://proceedings.neurips.cc/paper_files/paper/2023/file/7b97adeafa1c51cf65263459ca9d0d7c-Paper-Conference.pdf}.

\bibitem[Leviathan et~al.(2023)Leviathan, Kalman, and Matias]{leviathan2023fast}
Y.~Leviathan, M.~Kalman, and Y.~Matias.
\newblock Fast inference from transformers via speculative decoding.
\newblock In \emph{International Conference on Machine Learning}, pages 19274--19286. PMLR, 2023.

\bibitem[Li et~al.(2024)Li, Wei, Zhang, and Zhang]{eagle}
Y.~Li, F.~Wei, C.~Zhang, and H.~Zhang.
\newblock Eagle: speculative sampling requires rethinking feature uncertainty.
\newblock In \emph{Proceedings of the 41st International Conference on Machine Learning}, ICML'24. JMLR.org, 2024.

\bibitem[Liu et~al.(2024)Liu, Li, Lv, Liu, Zhu, and Hu]{liu2024parallel}
T.~Liu, Y.~Li, Q.~Lv, K.~Liu, J.~Zhu, and W.~Hu.
\newblock Parallel speculative decoding with adaptive draft length.
\newblock \emph{arXiv preprint arXiv:2408.11850}, 2024.

\bibitem[Liu et~al.(2023)Liu, Hu, Bailis, Cheung, Deng, Stoica, and Zhang]{liu2023online}
X.~Liu, L.~Hu, P.~Bailis, A.~Cheung, Z.~Deng, I.~Stoica, and H.~Zhang.
\newblock Online speculative decoding.
\newblock \emph{arXiv preprint arXiv:2310.07177}, 2023.

\bibitem[Miao et~al.(2024)Miao, Oliaro, Zhang, Cheng, Wang, Zhang, Wong, Zhu, Yang, Shi, et~al.]{miao2024specinfer}
X.~Miao, G.~Oliaro, Z.~Zhang, X.~Cheng, Z.~Wang, Z.~Zhang, R.~Y.~Y. Wong, A.~Zhu, L.~Yang, X.~Shi, et~al.
\newblock Specinfer: Accelerating large language model serving with tree-based speculative inference and verification.
\newblock In \emph{Proceedings of the 29th ACM International Conference on Architectural Support for Programming Languages and Operating Systems, Volume 3}, pages 932--949, 2024.

\bibitem[Narasimhan et~al.(2024)Narasimhan, Jitkrittum, Rawat, Kim, Gupta, Menon, and Kumar]{narasimhan2024faster}
H.~Narasimhan, W.~Jitkrittum, A.~S. Rawat, S.~Kim, N.~Gupta, A.~K. Menon, and S.~Kumar.
\newblock Faster cascades via speculative decoding, 2024.

\bibitem[Narasimhan et~al.(2025)Narasimhan, Jitkrittum, Rawat, Kim, Gupta, Menon, and Kumar]{narasimhan2025faster}
H.~Narasimhan, W.~Jitkrittum, A.~S. Rawat, S.~Kim, N.~Gupta, A.~K. Menon, and S.~Kumar.
\newblock Faster cascades via speculative decoding.
\newblock In \emph{The Thirteenth International Conference on Learning Representations}, 2025.
\newblock URL \url{https://openreview.net/forum?id=vo9t20wsmd}.

\bibitem[Narayan et~al.(2018)Narayan, Cohen, and Lapata]{xsum-emnlp}
S.~Narayan, S.~B. Cohen, and M.~Lapata.
\newblock Don't give me the details, just the summary! {T}opic-aware convolutional neural networks for extreme summarization.
\newblock In \emph{Proceedings of the 2018 Conference on Empirical Methods in Natural Language Processing}, Brussels, Belgium, 2018.

\bibitem[Oldham(2001)]{oldham2001combinatorial}
J.~D. Oldham.
\newblock Combinatorial approximation algorithms for generalized flow problems.
\newblock \emph{Journal of Algorithms}, 38\penalty0 (1):\penalty0 135--169, 2001.

\bibitem[Schuster et~al.(2022)Schuster, Fisch, Gupta, Dehghani, Bahri, Tran, Tay, and Metzler]{schuster2022confident}
T.~Schuster, A.~Fisch, J.~Gupta, M.~Dehghani, D.~Bahri, V.~Q. Tran, Y.~Tay, and D.~Metzler.
\newblock Confident adaptive language modeling.
\newblock In A.~H. Oh, A.~Agarwal, D.~Belgrave, and K.~Cho, editors, \emph{Advances in Neural Information Processing Systems}, 2022.

\bibitem[Sun et~al.(2024)Sun, Chen, Yang, Tian, and Chen]{sun2024triforce}
H.~Sun, Z.~Chen, X.~Yang, Y.~Tian, and B.~Chen.
\newblock Triforce: Lossless acceleration of long sequence generation with hierarchical speculative decoding.
\newblock \emph{arXiv preprint arXiv:2404.11912}, 2024.

\bibitem[Sun et~al.(2025)Sun, Mendlovic, Leviathan, Aharoni, Ro, Beirami, and Suresh]{sun2025block}
Z.~Sun, U.~Mendlovic, Y.~Leviathan, A.~Aharoni, J.~H. Ro, A.~Beirami, and A.~T. Suresh.
\newblock Block verification accelerates speculative decoding.
\newblock In \emph{The Thirteenth International Conference on Learning Representations}, 2025.
\newblock URL \url{https://openreview.net/forum?id=frsg32u0rO}.

\bibitem[Team et~al.(2024)Team, Riviere, Pathak, Sessa, Hardin, Bhupatiraju, Hussenot, Mesnard, Shahriari, Ram{\'e}, et~al.]{team2024gemma}
G.~Team, M.~Riviere, S.~Pathak, P.~G. Sessa, C.~Hardin, S.~Bhupatiraju, L.~Hussenot, T.~Mesnard, B.~Shahriari, A.~Ram{\'e}, et~al.
\newblock Gemma 2: Improving open language models at a practical size.
\newblock \emph{arXiv preprint arXiv:2408.00118}, 2024.

\bibitem[Wayne(1999)]{wayne1999generalized}
K.~D. Wayne.
\newblock \emph{Generalized maximum flow algorithms}.
\newblock Cornell University, 1999.

\bibitem[Wayne and Fleischer(1999)]{wayne1999faster}
K.~D. Wayne and L.~Fleischer.
\newblock Faster approximation algorithms for generalized flow.
\newblock In \emph{Proceedings of the tenth annual ACM-SIAM symposium on Discrete algorithms}, pages 981--982, 1999.

\bibitem[Xiao et~al.(2024)Xiao, Zhang, Ge, Ouyang, Ordonez, and Yu]{xiao2024parallelspec}
Z.~Xiao, H.~Zhang, T.~Ge, S.~Ouyang, V.~Ordonez, and D.~Yu.
\newblock Parallelspec: Parallel drafter for efficient speculative decoding.
\newblock \emph{arXiv preprint arXiv:2410.05589}, 2024.

\bibitem[Zhang et~al.(2024)Zhang, Wang, Li, Shou, Chen, Chen, and Mehrotra]{zhang2024draftverifylossless}
J.~Zhang, J.~Wang, H.~Li, L.~Shou, K.~Chen, G.~Chen, and S.~Mehrotra.
\newblock Draft \& verify: Lossless large language model acceleration via self-speculative decoding, 2024.
\newblock URL \url{https://arxiv.org/abs/2309.08168}.

\bibitem[Zhou et~al.(2023)Zhou, Lyu, Rawat, Menon, Rostamizadeh, Kumar, Kagy, and Agarwal]{zhou2023distillspec}
Y.~Zhou, K.~Lyu, A.~S. Rawat, A.~K. Menon, A.~Rostamizadeh, S.~Kumar, J.-F. Kagy, and R.~Agarwal.
\newblock Distillspec: Improving speculative decoding via knowledge distillation.
\newblock \emph{arXiv preprint arXiv:2310.08461}, 2023.

\end{thebibliography}
